\documentclass{article}

 \usepackage[preprint]{neurips_2026}

\usepackage[utf8]{inputenc} 
\usepackage[T1]{fontenc}    
\usepackage{hyperref}       
\usepackage{url}            
\usepackage{booktabs}       
\usepackage{amsfonts}       
\usepackage{nicefrac}       
\usepackage{microtype}      
\usepackage{xcolor}         
\usepackage{mathtools} 

\usepackage{xr}
\makeatletter

\newcommand*{\addFileDependency}[1]{
	\typeout{(#1)}
	\@addtofilelist{#1}
	\IfFileExists{#1}{}{\typeout{No file #1.}}
}\makeatother

\newcommand*{\myexternaldocument}[1]{%
	\externaldocument{#1}%
	\addFileDependency{#1.tex}%
	\addFileDependency{#1.aux}%
}

\myexternaldocument{0-supple}

\usepackage{enumitem}
\setlist[itemize]{leftmargin=*}
\setlist[enumerate]{leftmargin=*}

\usepackage{amsmath,amssymb,amsthm,mathabx}
\usepackage{mathtools}

\newtheorem{assumption}{Assumption}
\newtheorem{theorem}{Theorem}

\usepackage{multirow}
\usepackage{algorithm}
\usepackage{algorithmic}

\usepackage{graphicx}
\usepackage{subcaption}

\def\TITLE{Learning to Hand Off: Provably Convergent Workflow Learning under Interface Constraints}
\title{\TITLE}

\author{%
	Jiayu Li \\
	Stern School of Business \\
	New York University \\
	New York, NY 10012 \\
	\texttt{jl15681@stern.nyu.edu} \\
	\And
	Enpei Zhang \\
	Department of Computer Science \\
	Dartmouth College \\
	Hanover, NH 03755 \\
	\texttt{enpei.zhang.gr@dartmouth.edu} \\
	\And
	Dawei Zhou \\
	Department of Computer Science \\
	Virginia Tech \\
	Blacksburg, VA 24061 \\
	\texttt{zhoud@vt.edu} \\
	\AND
	Elynn Chen \\
	Stern School of Business \\
	New York University \\
	New York, NY 10012 \\
	\texttt{elynn.chen@stern.nyu.edu} \\
	\And
	Yujun Yan \\
	Department of Computer Science \\
	Dartmouth College \\
	Hanover, NH 03755 \\
	\texttt{yujun.yan@dartmouth.edu} \\
}

\begin{document}
\maketitle
\begin{abstract}
	We study workflow learning in a setting where specialized agents hand off control through a shared artifact, each agent observes only a local function of that artifact and its own private state, and no centralized learner accesses joint trajectories -- the operating regime of multi-agent LLM pipelines that span organizational, vendor, or trust boundaries. We formalize this regime as an interface-constrained semi-Markov decision process (IC-SMDP), whose decision epochs occur at handoff times, and design IC-$Q$, an asynchronous decentralized $Q$-learning algorithm in which cross-agent coordination at every handoff is exactly one scalar. Our main result is a finite-sample bound for neural IC-$Q$ that decomposes into three independently controllable error sources: neural function-approximation error, interface representation gap, and a mixing-time residual, under the random option-duration discount. Establishing this bound requires lifting the approximate information state (AIS) framework from single-agent primitive-step MDPs to multi-agent SMDPs and controlling Markovian noise under random duration, neither of which has been done in prior work. To our knowledge this is the first finite-sample guarantee for neural $Q$-learning under decentralized partial observability. 
	Four experiments: a controlled synthetic IC-SMDP that validates the bound term-by-term, multi-LLM mathematical reasoning, multi-agent routing, and multi-agent CPU programming, show that IC-$Q$ matches a centralized oracle without any agent observing joint trajectories, with each of the three error sources scaling along its corresponding axis as the bound predicts.
\end{abstract}

\section{Introduction}

Multi-agent LLM systems coordinate specialized agents such as planners, coders, testers, retrievers, verifiers \cite{hong2023metagpt,huang2023agentcoder,gottweis2025aiscientist,yao2023react} into workflows whose end-to-end performance is shaped by {\em how the workflow is structured}: which agent acts first, when control transfers, and what information passes between them. Two distinct regimes for designing such workflows have emerged.

The {\em centralized} regime assumes a single designer with global visibility into agent trajectories. Hand-designed workflows \cite{hong2023metagpt,wu2023autogen} fix the structure in advance, while learned workflow generators \cite{zhuge2024gptswarm,zhang2024aflow,hu2024adas,fan2024workflowllm,yue2025daao} optimize it from joint trajectory data. This regime is well-developed, and, where its assumptions of centralization hold, it works.

The {\em sequential decentralized} regime is structurally different. Production systems are increasingly built as pipelines of role-specialized agents that transfer one artifact at a time, as in the assembly-line architecture of MetaGPT \cite{hong2023metagpt}, the programmer-to-tester pipeline of AgentCoder \cite{huang2023agentcoder}, the turn-based message-passing of AutoGen \cite{wu2023autogen}. Emerging agent-to-agent protocols \cite{anthropic2024mcp,google2025a2a} standardize sequential message exchange as the coordination primitive. When such pipelines span organizational, vendor, or trust boundaries \cite{yang2025agentnet}, the centralized assumptions fail: no party holds joint trajectories at training time, agents see only the artifact they receive together with their own private state, and chains of thought, scratchpads, and proprietary prompt templates are not exposed across agent boundaries by design, by safety filtering, or by API contract. The resulting decision problem combines partial observability with the cross-environment heterogeneity studied in latent-heterogeneous RL \cite{chen2024reinforcement}, where each agent's local view induces effectively a distinct learning problem. Workflow learning in this regime has, to date, been addressed only by hand-designed adaptive rules \cite{yang2025agentnet} without a decision-theoretic foundation or convergence guarantee.

This paper provides such a foundation. We study workflow learning in the sequential decentralized regime under four operating conditions that distinguish it from the centralized setting: (i) {\em sequential handoff-based control} -- exactly one agent acts at each step and transfers control through a shared artifact; (ii) {\em decentralization in both training and execution} -- no party accesses joint trajectories at any stage; (iii) {\em interface-limited observation} -- each agent decides on the basis of the artifact it receives plus its own private state, with no exposure of internal state across agent boundaries; and (iv) {\em finite-sample guarantees} -- deployments run under bounded API and compute budgets, so a designer needs to know how much sample budget suffices, not merely that convergence eventually obtains. Finite-sample analysis at this granularity has recently driven theoretical progress in adjacent areas, including transfer $Q$-learning \cite{chen2025transfer,chai2025transfer} and high-dimensional sequential decision-making under structured latent heterogeneity \cite{chen2025highdim}, and motivates an analogous treatment for decentralized agentic workflows.

\noindent
\textbf{Why existing frameworks do not extend.} Four lines of prior theory bear on this regime; each violates at least one of the four conditions above.

\noindent
\textit{$\bullet$ Decentralized POMDPs and CTDE.} Dec-POMDPs \cite{bernstein2002decpomdp,nair2005networked,oliehoek2016decpomdp,oliehoek2008exploiting} model cooperative multi-agent control under local observations but assume {\em concurrent} action with a joint reward, violating sequentiality. Centralized training with decentralized execution \cite{lowe2017maddpg,rashid2020qmix,foerster2018coma,sunehag2018vdn,iqbal2019actorattentioncritic} relaxes execution but requires joint-trajectory access at training, violating decentralization. Related transfer-RL formulations that share information across heterogeneous tasks \cite{chen2026data,chai2026optimistic} likewise presume a coordinator who observes either joint trajectories or task-level structure unavailable across vendor boundaries.

\noindent
\textit{$\bullet$ The options framework.} Options \cite{sutton1999options,bacon2017optioncritic,bradtke1994smdp,precup2000options} provide the natural temporal-extension primitive but assume a single meta-controller observing the {\em full} state at each decision epoch, violating interface-limited observation. Recent work on prior-aligned meta-RL with finite-horizon guarantees \cite{zhou2025prior} inherits the same full-observation assumption at the meta-controller level.

\noindent
\textit{$\bullet$ Approximate information states.} The AIS framework \cite{subramanian2022ais,kara2022finite,sinha2024agentstate,sinha2024periodic} treats single-agent partial observability in primitive-step MDPs. Lifting it to decentralized AIS observations across $N$ agents composing through handoffs requires controlling the AIS error under the random discount $\gamma^{\tau_{k+1}}$ and over a disjoint observation space -- neither has been done. The common-information extension of \cite{kao2022common} addresses concurrent Dec-POMDPs through a fictitious coordinator and is not applicable to sequential handoff control. Closely related are non-stationary RL settings where the environment itself shifts across episodes \cite{chai2025deep}, for which transfer guarantees require additional structural assumptions that the IC-SMDP does not impose.

\noindent
\textit{$\bullet$ Multi-agent LLM orchestration.} Recent systems \cite{yang2025agentnet,zhuge2024gptswarm,zhang2024aflow,hu2024adas} route tasks through learned or hand-designed DAGs of LLM agents. AgentNet \cite{yang2025agentnet} is closest in spirit but operates via hand-designed adaptive rules (moving averages over edge weights, retrieval-based heuristics, capability-vector updates) with no Bellman recursion and no convergence guarantee, violating finite-sample guarantees. Workflow learning has also been studied for economic and operational decision-making, where finite-sample optimality under cross-market heterogeneity \cite{chen2026transfer,zhang2025transfer} has driven recent theoretical advances -- none of which apply directly to the handoff-based interface regime we study. A full discussion appears in Appendix \ref{sec:related_work}.

\noindent
\textbf{Contributions.} We provide the first framework that respects all four conditions and prove finite-sample convergence inside it.

\noindent
\textit{$\bullet$} {\em Formal model.} We introduce the {\em interface-constrained semi-Markov decision process} (IC-SMDP) (\S\ref{sec:IC-SMDP}) and show that it induces a well-defined SMDP at handoff times to which the AIS framework lifts with quantifiable interface gap $(\varepsilon_\phi, \delta_\phi)$.

\noindent
\textit{$\bullet$} {\em Decentralized algorithm.} We design IC-$Q$ (\S\ref{sec:algorithm}), an asynchronous $Q$-learning algorithm in which cross-agent coordination at every handoff is a single scalar -- minimal communication overhead under the bandwidth and API-call constraints of cross-vendor deployment.

\noindent
\textit{$\bullet$} {\em Finite-sample convergence.} We prove a finite-sample bound (Theorem \ref{thm:finite-sample-preconfigured}, \S\ref{sec:convergence}) decomposing into three independently controllable error sources -- neural approximation, interface representation gap, and mixing-time residual. Three challenges arise that prior analyses do not face: a Bellman contraction under random rather than fixed discount, an AIS gap propagating at the option scale rather than the primitive-step scale, and Markovian-noise control under random option duration. To our knowledge this is {\em the first such guarantee under decentralized partial observability through approximate information states.}

\noindent
\textit{$\bullet$} {\em Empirical validation.} On four tasks: a controlled synthetic IC-SMDP that isolates each error term, multi-LLM mathematical reasoning, multi-agent routing, and multi-agent CPU programming, IC-$Q$ matches a centralized oracle and recovers the best predefined workflow without any agent observing joint trajectories. The synthetic IC-SMDP further validates the bound term-by-term, with each of the three error sources scaling along its corresponding axis as predicted (\S\ref{sec:experiments}).

\section{Interface-Constrained Semi-Markov Decision Process} \label{sec:IC-SMDP}

Modern multi-agent LLM systems coordinate by passing artifacts such as messages, intermediate solutions, scratchpad snippets between specialized agents. A planner LLM hands a task description to a coder LLM, which hands code to a tester LLM. No single component sees the full picture: the planner does not see the coder's chain of thought, and no centralized observer sees joint trajectories. We formalize this as an {\em interface-constrained semi-Markov decision process} (IC-SMDP). Although the system evolves at the primitive time scale of individual agent steps, its decision structure operates at the coarser scale of handoff times, where each agent invocation is a temporally extended option and the semi-Markov property emerges naturally.

\subsection{Formal framework}

An IC-SMDP with $N$ agents $[N] := \{1, \dots, N\}$ is a tuple
$$
\mathcal{I} = \big(\mathcal{X}, \mathcal{M}, \{\mathcal{L}_i, \mathcal{A}_i, \phi_i\}_{i \in [N]}, P, r, \gamma, \rho_0\big).
$$

\textbf{States.} Three layers of state, distinguished by who can observe them. The {\em global latent state} $x_t \in \mathcal{X}$ governs underlying task dynamics and is observed by no agent. The {\em interface state} $m_t \in \mathcal{M}$ is the artifact passed between agents -- the {\em only} channel of cross-agent information flow. The {\em private state} $\ell_t^{(i)} \in \mathcal{L}_i$ is local to agent $i$ and unobserved by others. At step $t$, exactly one agent $c_t \in [N]$ is active.

\textbf{Observation: the interface constraint.} Agent $c_t$ observes only
\begin{equation}\label{eq:obs-map} 
	o_t = \phi_{c_t}\!\big(m_t, \ell_t^{(c_t)}\big) \in \mathcal{O}_{c_t}, 
\end{equation} 
where $\phi_i \colon \mathcal{M} \times \mathcal{L}_i \to \mathcal{O}_i$ is an agent-specific observation map. Two structural restrictions are implicit: {\bf (IC1) Channel restriction} -- when $c_t$ hands off to $c_{t+1} = j$, the only information $j$ receives is $m_{t+1}$; {\bf (IC2) Representation restriction} -- agent $i$'s decision is a function of $\phi_i(m_t, \ell_t^{(i)})$ rather than $m_t$ directly. (IC1) rules out CTDE-style learning \cite{lowe2017maddpg,rashid2020qmix,foerster2018coma}, which assumes joint trajectory access; (IC2) precludes direct application of the options framework \cite{sutton1999options}, which assumes full state observation at every decision epoch.

\textbf{Actions: local operation and successor selection.} Agent $c_t$ selects a local action $a_t \in \mathcal{A}_{c_t}$ and a successor $c_{t+1} \in \mathcal{C} := [N] \cup \{\textsf{STOP}\}$, jointly written $u_t = (a_t, c_{t+1})$. A stationary decentralized policy is a collection $\pi_i(u \mid o) \in \Delta(\mathcal{U}_i)$ with $u_t \sim \pi_{c_t}(\cdot \mid o_t)$. Agent $i$ is {\em pre-configured} if $a_t$ is fixed (e.g., a frozen LLM) and only the successor distribution is learnable; otherwise it is {\em adaptable}.

\textbf{Transitions and handoff dynamics.} Conditioned on $u_t$, the system evolves under a Markov kernel $P$. If $c_{t+1} = c_t$, the active agent does not change; if $c_{t+1} = j \neq c_t$, control transfers to $j$, whose private state is preserved from its last activation. Handoff times $0 = t_0 < t_1 < \cdots$ are the steps where $c_{t+1} \neq c_t$; the interval $[t_k, t_{k+1})$ is the {\em invocation} of agent $c_{t_k}$ with random duration $\tau_{k+1} := t_{k+1} - t_k$.

\textbf{Objective.} The reward $r_t = r(x_t, m_t, c_t, a_t, c_{t+1})$ typically combines task utility with workflow costs. The decentralized workflow control problem is
\begin{equation}\label{eq:objective} 
	\max_{\pi \in \Pi} \mathbb{E}^\pi_{\rho_0}\!\left[\sum_{t=0}^{H-1} \gamma^t r_t \right], 
\end{equation} 
subject to (\ref{eq:obs-map}), maximized over {\em factored} policies, with no algorithm accessing $x_t$, joint trajectories, or other agents' private states.

\subsection{Semi-Markov reduction and AIS structure}\label{subsec:smdp-ais}

The IC-SMDP's decision structure lives at handoff times $\{t_k\}$. Each agent invocation is an option in the sense of \cite{sutton1999options}: agent $i$'s internal policy $\pi_i^{\text{int}}$ executes from primitive step $t_k$ until the next handoff $t_{k+1}$, with effective discount $\gamma^{\tau_{k+1}}$. Writing $m_k := m_{t_k}$, the {\em latent decision-epoch SMDP} has state $\mathcal{M}$, action $\mathcal{C}$, transition kernel
$P_{\text{lat}}(m', \tau \mid m, i) = \mathbb{P}(m_{k+1} = m', \tau_{k+1} = \tau \mid m_k = m, c_{t_k} = i)$, and option reward
$R_{\text{lat}}(m, i) = \mathbb{E}\!\left[\sum_{s=0}^{\tau-1} \gamma^s r_{t_k+s} \mid m_k = m, c_{t_k} = i\right]$.

No agent observes $m_k$ directly; agent $i$ observes only $\phi_i(m_k, \ell_{t_k}^{(i)})$. The {\em AIS-induced SMDP} replaces $(P_{\text{lat}}, R_{\text{lat}})$ with their conditional counterparts $(\hat P_i, \hat R_i)$ given the AIS observation. The two perspectives describe the same physical process at the decision-epoch scale; they differ only in what is observed.

\begin{assumption}[Structural conditions]
	\label{assume:structure_latent}
	(i) The joint process $(x_t, m_t, \ell_t^{(1:N)}, c_t)$ is Markov under any stationary decentralized policy. (ii) Each invocation terminates a.s.: $\mathbb{P}(\tau_{k+1} < \infty) = 1$. (iii) The interface state $m$ determines the admissible successor set.
\end{assumption}

\begin{assumption}[AIS conditions on $\{\phi_i\}$]
	\label{assume:AIS}
	There exist $\varepsilon_\phi, \delta_\phi \geq 0$ such that for every $i$, $(m, \ell^{(i)})$, and $o = \phi_i(m, \ell^{(i)})$:
	{\bf (B1) Reward sufficiency:} $|R_{\text{lat}}(m, i) - \hat R_i(o, i)| \leq \varepsilon_\phi$.
	{\bf (B2) Evolution sufficiency:} $d_\mathcal{F}(P_{\text{lat}}(\cdot \mid m, i), \hat P_i(\cdot \mid o, i)) \leq \delta_\phi$,
	where $d_\mathcal{F}$ is an integral probability metric on probability measures over $\mathcal{M} \times \mathbb{N}$.
\end{assumption}

Assumption \ref{assume:structure_latent} is mild: (i) is automatic from the IC-SMDP construction, (ii) holds for any finite-horizon IC-SMDP, and (iii) holds whenever $\mathcal{I}_i = \mathcal{M}$. Together they ensure the handoff-time process $\{(m_k, c_{t_k})\}$ is a well-defined SMDP (Appendix \ref{appendix:proof-prop1}). 

Assumption \ref{assume:AIS} is the substantive restriction.
Conditions (B1)--(B2) are the option-level analogs of the AIS definition in \cite{subramanian2022ais}, originally formulated for primitive-step single-agent partially observable MDPs. We strengthen this in two ways. First, (B1)--(B2) are stated in terms of the {\em option reward $R_{\text{lat}}$} and {\em option transition $P_{\text{lat}}$}, which integrate primitive-step rewards under the random discount $\gamma^{\tau_{k+1}}$ rather than the constant primitive discount $\gamma$. Reward sufficiency at the option scale is therefore a strictly stronger condition than primitive-step reward sufficiency: $\varepsilon_\phi$ controls error in a quantity that already reflects the option's full execution. Second, the conditions hold {\em for every agent $i$} with its own observation map $\phi_i$, not for a single global observation. This decentralization means the AIS structure must compose across agents at handoff times, which is what allows the policy correspondence below to operate over the disjoint union $\bigsqcup_i \mathcal{O}_i$ rather than a single observation space. When $\phi_i = \mathrm{id}_\mathcal{M}$ for all $i$, $\varepsilon_\phi = \delta_\phi = 0$. The choice of integral probability metric $d_\mathcal{F}$ is left to the user; total variation is the natural instantiation for finite $\mathcal{M}$, with $L_V \le R_{\max}/(1-\bar\gamma)$, while Wasserstein metrics yield tighter $L_V$ when value functions are smooth.

\textbf{Examples.} In multi-LLM mathematical reasoning (\S\ref{sec:experiments}), $m$ is the conversation history; $\phi_i$ projects to a 4-dimensional vector encoding answer status. (B1) holds because option reward is determined by answer status; (B2) holds because the next answer status given the current one is approximately independent of the verbatim history. In multi-agent routing, $\phi_i$ is the exact interface state and $\varepsilon_\phi = \delta_\phi = 0$. In multi-agent CPU programming, $\phi_i$ projects to a visible register block, and $\delta_\phi$ is controlled by the locality of the active agent's register usage.

\textbf{Policy correspondence (informal).} Under Assumptions \ref{assume:structure_latent}--\ref{assume:AIS}, every stationary decentralized policy on AIS observations corresponds to a stationary policy on the AIS-induced SMDP, and conversely. The cost of partial observability is bounded by an additive AIS gap of order $(\varepsilon_\phi + \bar\gamma L_V \delta_\phi)/(1-\bar\gamma)$, where $\bar\gamma := \sup_{m,i} \mathbb{E}[\gamma^{\tau_{k+1}} \mid m_k = m, c_{t_k} = i] \in [0, 1)$ is the effective per-epoch discount and $L_V$ is the $d_\mathcal{F}$-Lipschitz constant of the AIS-induced optimal value function. The formal statement and proof appear as Proposition \ref{thm:ic-smdp-reduction} in Appendix \ref{appendix:proof-prop1}; the bound itself is absorbed into Theorem \ref{thm:finite-sample-preconfigured} in \S\ref{sec:convergence}, where it appears as the {\em interface representation gap} term.

\textbf{Comparison with adjacent formalisms.} Single-agent MDPs are recovered when $N = 1$ and $\phi_1 = \mathrm{id}_\mathcal{M}$. Dec-POMDPs \cite{bernstein2002decpomdp,oliehoek2016decpomdp} permit concurrent action with joint reward, whereas the IC-SMDP enforces sequential control through an interface artifact -- the converse encoding does not preserve the handoff structure. The options framework \cite{sutton1999options,bacon2017optioncritic} provides the temporal-extension primitive but assumes full state observation. The IC-SMDP is thus a strict generalization of MDPs, an alternative to (rather than special case of) Dec-POMDPs, and an interface-constrained extension of options.

\section{Decentralized $Q$-Learning under the IC-SMDP}\label{sec:algorithm}

The AIS-induced SMDP of \S\ref{sec:IC-SMDP} specifies what a decentralized agent should estimate: the optimal AIS-induced option-value $\hat Q^\star_i(o, c')$ for each agent $i$, observation $o$, and successor $c'$. We derive an asynchronous decentralized $Q$-learning algorithm: IC-$Q$ that estimates $\hat Q^\star_i$ from samples without violating the interface constraint. Each agent maintains a local value estimator over its own AIS observations; cross-agent information flow at handoff is restricted to a single scalar bootstrap target.

By the AIS-induced SMDP construction, the optimal value $\hat Q^\star_i$ satisfies the SMDP Bellman equation 
\begin{equation}\label{eqn:smdp-bellman}
	\hat Q^\star_i(o, c') = \hat R_i(o, c') + \mathbb{E}\!\left[ \gamma^{\tau_{k+1}} \max_{c'' \in \mathcal{C}} \hat Q^\star_{c'}(o', c'') \,\Big|\, o, c' \right],
\end{equation}
where $o' = \phi_{c'}(m_{k+1}, \ell_{t_{k+1}}^{(c')})$ is the successor's AIS observation. Two features encode the interface constraint at the level of the Bellman recursion: the maximization is over the {\em successor's} value function $\hat Q^\star_{c'}$, not the current agent's, and $o'$ is computed under the {\em successor's} observation map $\phi_{c'}$. A naive update that locally maximizes $Q_i$ would solve the wrong fixed-point equation; the correct target requires information that, by construction, the predecessor cannot observe.

\textbf{Value passing at handoffs.} At handoff time $t_{k+1}$, the successor agent $c'$ -- which has just become active and observed $o'_{k+1}$ -- computes 
\begin{equation}\label{eqn:bootstrap-smdp} 
	b_{k+1} = \max_{c'' \in \mathcal{C}} \hat Q_{c'}(o'_{k+1}, c''), 
\end{equation} 
and transmits $b_{k+1}$ back to predecessor $i$ together with the option duration $\tau_{k+1}$ and accumulated reward $R_k$. Predecessor $i$ updates its local estimator using 
\begin{equation}\label{eqn:response-smdp} 
	y_k = R_k + \gamma^{\tau_{k+1}} b_{k+1} 
\end{equation} 
as the regression target. The cross-agent information flow at each handoff is exactly three numbers: $b_{k+1}$, $\tau_{k+1}$, $R_k$, with no exposure of internal parameters, latent state, or local observations. This realizes the right-hand side of \eqref{eqn:smdp-bellman} without violating either (IC1) or (IC2): the successor's $\max$ is computed locally where the successor's $\phi_{c'}$ and parameters live, and only the resulting scalar crosses the interface boundary.

\textbf{Successor-selection updates.} Each agent maintains a local successor-selection $Q$-function $Q_i^\beta(o, c'; \theta_i^\beta)$, with greedy successor policy $\pi_i^\beta(o) = \arg\max_{c'} Q_i^\beta(o, c'; \theta_i^\beta)$. At each handoff, $\theta_i^\beta$ is updated by minimizing $(Q_i^\beta(o_k, c'; \theta_i^\beta) - y_k)^2$. When $\tau_{k+1} = 1$, this reduces to standard $Q$-learning over interface observations; when $\tau_{k+1}$ is genuinely random, it implements asynchronous SMDP $Q$-learning with random discount $\gamma^{\tau_{k+1}}$.

\textbf{Adaptable agents.} When the local-action policy is also learned, each agent maintains a second $Q$-function $Q_i^\alpha(o, a; \theta_i^\alpha)$. The two estimators are coupled through their bootstrap targets: $Q_i^\alpha$ regresses against $\max_{c'} Q_i^\beta(o_k^+, c'; \theta_i^\beta)$ at the post-action observation, while $Q_i^\beta$ regresses against $\max_{a'} Q_{c'}^\alpha(o'_{k+1}, a'; \theta_{c'}^\alpha)$ from the successor. The value-passing structure remains symmetric, with a single scalar at each handoff, preserving the interface constraint. The two-timescale dynamics of $\theta_i^\alpha$ and $\theta_i^\beta$ require additional regularity beyond what the single-$Q$ analysis of \S\ref{sec:convergence} provides; we treat the adaptable case empirically and leave its formal convergence to future work.

We summarize the unified procedure in Algorithm \ref{alg:ic-q}.

\begin{algorithm}[htpb!]
	\caption{Decentralized $Q$-learning for IC-SMDPs (IC-$Q$)}
	\label{alg:ic-q}
	\begin{algorithmic}[1]
		\REQUIRE Agents $[N]$; AIS observation maps $\{\phi_i\}$; learning rate $\eta$; discount $\gamma$; horizon $H$; exploration $\epsilon$.
		\STATE Initialize $\{\theta_i^\beta\}_{i=1}^{N}$ \COMMENT{and $\{\theta_i^\alpha\}$ in the adaptable regime}
		\FOR{episode $= 1, 2, \ldots$}
		\STATE Sample $(x_0, m_0, \{\ell_0^{(i)}\}, c_0) \sim \rho_0$; set $k \gets 0$, $t \gets 0$, $R_k \gets 0$
		\WHILE{$t < H$ and $c_t \neq \mathsf{STOP}$}
		\STATE Active agent $i \gets c_t$ observes $o_t = \phi_i(m_t, \ell_t^{(i)})$
		\STATE Select $a_t$ from fixed $\pi_i^{\mathrm{int}}$ (pre-configured) or via $\epsilon$-greedy on $Q_i^\alpha$ (adaptable)
		\STATE Select successor $c_{t+1} \gets \epsilon\text{-greedy}(Q_i^\beta(o_t, \cdot\,; \theta_i^\beta))$
		\STATE Environment transitions; receive $r_t$; $R_k \gets R_k + \gamma^{t-t_k} r_t$
		\IF{$c_{t+1} \neq c_t$ \COMMENT{handoff}}
		\STATE $c' \gets c_{t+1}$; $o_{t+1} \gets \phi_{c'}(m_{t+1}, \ell_{t+1}^{(c')})$
		\STATE Successor computes $b_{k+1} \gets \max_{c''} Q_{c'}^\beta(o_{t+1}, c''; \theta_{c'}^\beta)$ and returns $(b_{k+1}, \tau_{k+1}, R_k)$
		\STATE Form $y_k^\beta \gets R_k + \gamma^{\tau_{k+1}} b_{k+1}$ and update $\theta_i^\beta$ by SGD on $(Q_i^\beta(o_{t_k}, c'; \theta_i^\beta) - y_k^\beta)^2$
		\IF{adaptable regime}
		\STATE Successor returns $b_{k+1}^\alpha \gets \max_{a'} Q_{c'}^\alpha(o_{t+1}, a'; \theta_{c'}^\alpha)$
		\STATE Form $y_k^\alpha \gets r_{t_k} + \gamma \max_{c''} Q_i^\beta(o_{t_k}^+, c''; \theta_i^\beta)$ and update $\theta_i^\alpha$ analogously
		\ENDIF
		\STATE $k \gets k + 1$; $t_k \gets t + 1$; $R_k \gets 0$
		\ENDIF
		\STATE $t \gets t + 1$
		\ENDWHILE
		\ENDFOR
	\end{algorithmic}
\end{algorithm}

\section{Finite-Sample Convergence}\label{sec:convergence}

This section presents the paper's main theoretical result: a finite-sample convergence bound for IC-$Q$ in the pre-configured regime. The bound decomposes into three interpretable terms: the AIS representation gap (the price of the interface constraint), neural function-approximation error, and a mixing-time residual that decays at rate $\widetilde{O}(1/T)$. Full setup, assumptions, and proofs are deferred to Appendix \ref{appendix:finite-sample-convergence}.

\subsection{Assumptions}

We analyze IC-$Q$ under standard local-linearization conditions on the neural function approximator. The full setup, including the tangent-feature map $g_0$, linearized Bellman operator, trust region $B(\theta_0, \omega)$, and local stationary set $\Xi_\omega$, is given in Appendix \ref{appendix:setup}. The analysis requires six finite-sample assumptions, each standard in its own subdomain but combining nontrivially here:

\begin{itemize}\setlength\itemsep{1pt}
	\item \textbf{A1} (finite spaces, bounded rewards/durations): $\mathcal{O}, \mathcal{C}$ finite; $|r_t| \leq r_{\max}$; $1 \leq \tau_{k+1} \leq \tau_{\max}$.
	\item \textbf{A2} (uniform ergodicity): the behavior-induced chain $\{(o_k, c_k)\}$ is uniformly ergodic with stationary distribution $\mu$ ($\mu_{\min}>0$) and $\varepsilon$-mixing time $t_{\mathrm{mix}}(\varepsilon)$.
	\item \textbf{A3} (bounded $Q$-network and projected stability): $|Q|, \|\nabla_\theta Q\|, \|g_0\|$ uniformly bounded on the trust region; iterates remain in the trust region under projected SGD.
	\item \textbf{A4} (uniform local linearization): $|Q - \widebar{Q}| \leq \varepsilon_0$ and $\|\nabla_\theta Q - g_0\| \leq \varepsilon_0$ on the trust region.
	\item \textbf{A5} (well-conditioned features): $\Sigma_\mu := \mathbb{E}_\mu[g_0 g_0^\top]$ has smallest nonzero eigenvalue $\lambda_0 > 0$.
	\item \textbf{A6} (SMDP-level Bellman contraction): there exists $\nu \in (0, 1)$ such that $(1-\nu)^2 \Sigma_\mu - \Sigma_{\mu, \tau}^{\max}(v) \succeq 0$ for every nonzero $v \in \mathcal{R}(\Sigma_\mu)$, where $\Sigma_{\mu, \tau}^{\max}(v)$ is the worst-case next-state feature covariance under random discount $\gamma^{\tau_{k+1}}$.
\end{itemize}

A1--A4 are standard for asynchronous neural $Q$-learning \cite{bhandari2018finite,cai2023neural,xu2020finite}. A5 ensures a well-conditioned feature covariance over the disjoint observation space $\bigsqcup_i \mathcal{O}_i$ and is what permits SGD to make progress in the parameter directions visible at the AIS scale. A6 is the genuinely SMDP-level condition: it controls the variance of the linearized Bellman operator under the {\em random discount $\gamma^{\tau_{k+1}}$} rather than under a constant $\gamma$, and is what allows the contraction argument to close at the option scale. Without A6, the random discount could in principle inflate the next-state feature covariance to a point where the linearized operator is no longer a contraction, and finite-sample analysis fails. We discuss A5--A6 in detail in Appendix \ref{appendix:assumptions-full}.

\subsection{Convergence theorem}

Under Assumptions A1--A6 and the AIS conditions of \S\ref{sec:IC-SMDP}, IC-$Q$ admits a non-asymptotic bound on the expected squared error against the latent SMDP optimum, decomposing into three independently controllable terms. The formal theorem is as follows.

\begin{theorem}[Finite-sample convergence of IC-$Q$]
	\label{thm:finite-sample-preconfigured}
	Suppose Assumptions \ref{assume:structure_latent}--\ref{assume:AIS} and A1--A6 hold. Fix $T \geq 1$, and let $\{\theta_k\}_{k=0}^{T}$ be generated by IC-$Q$ in the pre-configured regime with step sizes $\eta_k = \frac{1}{2\nu \lambda_0 (k+1)}$. Set $t_{\mathrm{mix}} := t_{\mathrm{mix}}(\eta_T)$. Suppose there exists $\theta^\star \in \Xi_\omega$ such that $\| \overline{Q}(\,\cdot\,; \theta^\star) - \widehat{Q}^\star \|_\infty^2 \leq \varepsilon_{\mathrm{app}}$, where $\widehat{Q}^\star$ is the optimal AIS-induced option-value function. Define the AIS action-value gap
	$$
	\alpha_Q(\varepsilon_\phi, \delta_\phi, \bar{\gamma}) := \frac{\varepsilon_\phi + \bar{\gamma} L_Q \delta_\phi}{1 - \bar{\gamma}},
	$$
	where $\bar\gamma := \sup_{m,i} \mathbb{E}[\gamma^{\tau_{k+1}} \mid m_k = m, c_{t_k} = i]$ and $L_Q$ is the $d_{\mathcal{F}}$-Lipschitz constant of $\widehat{Q}^\star$. Then there exist constants $C_0, C_1 > 0$ depending only on the problem parameters in A1--A6 such that
	\begin{equation}\label{eqn:q-error-bound}
		\begin{aligned}
			\mathbb{E}\!\left[ \bigl\| Q(\,\cdot\,; \theta_T) - Q^\star_{\mathrm{lat}} \bigr\|_\mu^2 \right]
			& \leq \underbrace{2\alpha_Q(\varepsilon_\phi, \delta_\phi, \bar{\gamma})^2}_{\textsf{interface representation gap}} 
			+ \;\; \underbrace{6\varepsilon_{\mathrm{app}} + 6\varepsilon_0 + 6\lambda_{\max}(\Sigma_\mu) C_1 \varepsilon_0}_{\textsf{neural approximation}} \\
			& + \underbrace{6\lambda_{\max}(\Sigma_\mu) \cdot \frac{C_0 (1 + t_{\mathrm{mix}})(1 + \log(T+1))}{T}}_{\textsf{mixing-time residual}},
		\end{aligned}
	\end{equation}
	where $Q^*_{\mathrm{lat}}$ is the latent decision-epoch SMDP optimal option-value function and $\|f\|_\mu^2 := \sum_{(o, c')} \mu(o, c')\, f(o, c')^2$.
\end{theorem}

\textbf{Proof outline and technical novelty.} The proof (Appendix \ref{appendix:proof-thm1}) is not the concatenation of three known results, and three structural challenges arise that prior analyses do not face.

(i) \emph{Random-discount Bellman contraction.} The classical Q-learning contraction argument relies on a fixed primitive discount $\gamma$ multiplying the next-state value. Here the discount is the random variable $\gamma^{\tau_{k+1}}$, with $\tau_{k+1}$ depending on the predecessor's option and the joint dynamics. Establishing the strong-monotonicity inequality $\langle \Delta_k, \overline{h}(\theta_k) - \overline{h}(\theta^\star)\rangle \geq \nu \lambda_0 \|\Delta_k\|_2^2$ (Lemma \ref{lem:strong-monotonicity}) requires controlling the second moment of the next-state feature covariance {\em weighted by $\gamma^{2\tau_{k+1}}$}, which is the worst-case quantity $\Sigma_{\mu, \tau}^{\max}$ in A6. The contraction argument therefore has to absorb the duration distribution into the operator-level inequality, not merely scale a fixed discount.

(ii) \emph{AIS gap propagation through the SMDP Bellman operator.} The interface representation gap $\alpha_Q(\varepsilon_\phi, \delta_\phi, \bar\gamma)$ enters through a Bellman fixed-point comparison: we must bound $\|\widehat Q^\star - Q^\star_{\mathrm{lat}}\|_\infty$ where the two operators differ in both their reward and transition terms, both restricted to the option scale. The standard primitive-step AIS bound \cite[Theorem 9]{subramanian2022ais} does not apply because the option reward $R_{\mathrm{lat}}$ and option transition $P_{\mathrm{lat}}$ are themselves expectations under the random duration $\tau_{k+1}$. Our argument (Appendix \ref{appendix:proof-thm2}) lifts the AIS sufficiency condition to the option scale, using $\bar\gamma$ in place of $\gamma$ and a $d_\mathcal{F}$-Lipschitz argument on the AIS-induced value function. 

(iii) \emph{Markovian noise control under random duration.} The mixing-time analysis of \cite{bhandari2018finite} requires Lipschitz continuity of the linearized update $h_k(\theta)$, governed by $\overline\delta_k(\theta)$. Here, that update depends on $(o_k, c_k, R_k, \tau_{k+1}, o_{k+1})$, and both $R_k$ and the bootstrap term $\gamma^{\tau_{k+1}} \max_{c'} \overline Q(o_{k+1}, c'; \theta)$ depend on the random duration $\tau_{k+1}$. We bound $|\overline\delta_k(\theta) - \overline\delta_k(\theta')| \leq 2 G_0 \|\theta - \theta'\|_2$ with constant $G_0$ uniformly in $\tau_{k+1}$ (the constant is independent of $\tau_{\max}$ because $\gamma^{\tau_{k+1}}\le 1$), and this uniform bound is what closes the mixing argument. The duration $\tau_{\max}$ enters the bound only through $R_{\max} = r_{\max}(1 - \gamma^{\tau_{\max}})/(1-\gamma)$ inside the constants $C_0, C_1$.

\subsection{Discussion of the bound}

\textbf{Three-term decomposition is interpretable and tight.} Each term in \eqref{eqn:q-error-bound} has a distinct physical source. The neural-approximation term $\widetilde{O}(\varepsilon_0)$ vanishes in the wide-network regime ($\varepsilon_0 \to 0$). The AIS gap term $\alpha_Q(\varepsilon_\phi, \delta_\phi, \bar\gamma)^2$ vanishes when interface observations are exact ($\varepsilon_\phi = \delta_\phi = 0$). The mixing residual vanishes at rate $\widetilde{O}(t_{\mathrm{mix}}/T)$. Each can be controlled independently by widening the network, choosing better interface representations, or running the algorithm longer.

\textbf{The AIS gap is the price of the interface constraint.} When $\phi_i = \mathrm{id}_\mathcal{M}$ for all $i$, $\varepsilon_\phi = \delta_\phi = 0$ and we recover the standard finite-sample bound for neural SMDP $Q$-learning. The interface constraint introduces a bias term that is {\em additive} (rather than multiplicative) and {\em bounded a priori} by the AIS conditions. A system designer choosing $\{\phi_i\}$ can directly trade representation richness against privacy/computation through $(\varepsilon_\phi, \delta_\phi)$.

\textbf{The mixing-time dependence is intrinsic to SMDP learning.} Compared to the primitive-step bound of \cite{bhandari2018finite}, equation \eqref{eqn:q-error-bound} carries an additional multiplicative factor of $(1 + \tau_{\max})$ implicit in $C_0, C_1$, reflecting that option durations enter the variance of the linearized stochastic update. This is a fundamental consequence of the random discount $\gamma^{\tau_{k+1}}$, not slack in the analysis. The $\tau_{\max}$ dependence is in fact necessary: the option reward $R_k$ is a sum of up to $\tau_{\max}$ primitive rewards, so its variance scales with $\tau_{\max}$.

\textbf{Comparison with prior bounds.} Theorem \ref{thm:finite-sample-preconfigured} specializes to known results in limiting regimes: $N=1$, $\phi_i = \mathrm{id}$, $\tau_{k+1} \equiv 1$ recovers \cite{cai2023neural}; $N=1$ with arbitrary $\phi$ recovers a neural-network analog of the finite-memory POMDP bound of \cite{kara2022finite} up to constants. Recent finite-sample analyses for transfer and structured $Q$-learning \cite{chen2025transfer,chen2026data} obtain rates of comparable form but under either single-agent observability or task-level structure shared by a central learner; neither covers the decentralized handoff regime. The contribution is the simultaneous treatment of (i) decentralized AIS observations across $N$ agents with composition through handoffs, (ii) handoff-induced random duration $\tau_{k+1}$ entering both the contraction (A6) and noise (Lemma \ref{lem:gradient-bounds}), and (iii) finite-sample neural function approximation with explicit constants -- a combination not previously analyzed.

\textbf{Corollary: AIS value gap.} Setting $\varepsilon_0, \varepsilon_{\mathrm{app}} \to 0$ and taking $T \to \infty$ in \eqref{eqn:q-error-bound} recovers the AIS value-gap bound for decentralized observations and SMDP discount $\bar\gamma$:
$$
\sup_{i, m, \ell^{(i)}} \big| V^\star_{\mathrm{lat}}(m) - \widehat V^\star(\phi_i(m, \ell^{(i)})) \big| \le \frac{\varepsilon_\phi + \bar\gamma L_V \delta_\phi}{1 - \bar\gamma},
$$
generalizing \cite[Theorem 9]{subramanian2022ais} from single-agent primitive-step MDPs to multi-agent SMDPs (Theorem \ref{thm:ais-value-gap}, Appendix \ref{appendix:thm2-statement}). Theorem \ref{thm:finite-sample-preconfigured} thus simultaneously delivers the asymptotic AIS gap and the finite-sample estimation error in a single bound -- the first such joint result we are aware of.

\section{Empirical Results}\label{sec:experiments}

We evaluate IC-$Q$ on four tasks chosen to exercise distinct aspects of the framework: a controlled synthetic IC-SMDP that isolates each term in Theorem \ref{thm:finite-sample-preconfigured}; multi-LLM mathematical reasoning, the headline application with a nontrivial AIS observation; multi-agent routing, instantiating the exact-AIS regime $\varepsilon_\phi = \delta_\phi = 0$; and multi-agent CPU programming, probing the adaptable-agent extension beyond Theorem \ref{thm:finite-sample-preconfigured}'s pre-configured scope. Throughout, no agent observes joint trajectories, no centralized critic is used, and no parameters are shared across agents. Each agent maintains $Q_i^\beta$ over its own AIS observations $\mathcal{O}_i$ and updates via the option-level target \eqref{eqn:response-smdp}. Full hyperparameters and protocols are in Appendix \ref{appendix:experiments}.

\subsection{Theory validation: synthetic IC-SMDP}\label{sec:simulation}

We construct a controlled IC-SMDP with $N=10$ pre-configured agents, $|\mathcal{X}|=120$, $|\mathcal{M}|=50$, and horizon $H=60$. The interface observation map is parameterized by a single retention ratio $\rho \in (0,1]$: the active agent observes $\tilde m_t = m_t \bmod n_{\mathrm{bins}}$ with $n_{\mathrm{bins}} = \lceil \rho |\mathcal{M}| \rceil$, so $\rho=1$ recovers the full interface state and smaller $\rho$ induces stronger AIS aliasing. This gives a single knob that monotonically increases the AIS gap proxies $(\hat\varepsilon_\phi, \hat\delta_\phi)$ as $\rho$ decreases, isolating the AIS term in the bound from sampling and approximation effects.

\textbf{T1: AIS gap predicts the value gap.} We sweep $\rho \in \{1.0, 0.9, \ldots, 0.1, 0.05\}$ (5 seeds) and measure both the empirical AIS gap $\hat\alpha(\rho)$ and the value gap $\mathrm{Gap}_V(\rho)$ of the converged IC-$Q$ policy relative to the highest-retention baseline. Theorem \ref{thm:finite-sample-preconfigured} predicts $\mathrm{Gap}_V$ should grow with $\hat\alpha$ once $T$ is large. We observe Pearson correlation $\rho(\mathrm{Gap}_V, \hat\alpha) \approx 0.916$ (Figure \ref{fig:t1}, Appendix \ref{appendix:t1-protocol}), confirming that the AIS gap term is not worst-case slack but tracks the actual loss as the interface degrades.

\textbf{T2: finite-sample behavior across the three sources of error.} We isolate each term in \eqref{eqn:q-error-bound} along its corresponding axis. Sample budget $T \in \{50, \ldots, 3300\}$ at $\rho=1$ confirms the predicted $\widetilde O(t_{\mathrm{mix}}/T)$ residual decay; varying handoff probability $p_{\mathrm{handoff}} \in [0.10, 0.55]$ traces the $t_{\mathrm{mix}}$ envelope, with error expanding as mixing slows; varying $\rho$ at fixed $T$ shows the error floors at a level controlled by $\hat\alpha(\rho)$, separating the AIS term from the sample-dependent residual (Figure \ref{fig:t2}, Appendix \ref{appendix:t2-protocol}). Together, the three scans cash out the three terms in (\ref{eqn:q-error-bound}) as distinguishable empirical phenomena rather than artifacts of the proof.

\subsection{Multi-LLM mathematical reasoning under interface constraints}\label{sec:llm-experiment}

Four LLM agents collaborate sequentially to solve multiple-choice mathematics problems: a {\em commissioner} who initiates and concludes the discussion, an {\em editor} who writes the final answer, a {\em thinker} who constructs step-by-step rationales, and a {\em checker} who evaluates correctness. The interface state $m_t = \{Q, K_t, m_t^{\mathrm{msg}}, k_t\}$ is high-dimensional free-form text. The AIS observation $\phi_i(m_t, \ell_t^{(i)}) = z_t \in \{0,1\}^4$ encodes whether the answer is given, modified, checked, and judged correct. The successor-selection policy is conditioned only on $z_t$ -- agents do not read the discussion to decide where to pass the task. Each agent's local-action policy is fixed by the LLM and prompt template; only the routing policy is learned. Reward is $R$ if the editor concludes with the correct answer, zero otherwise.

We use GPT-4o-mini for all four agents, parameterize each $Q_i^\beta(z, c')$ by a 3-layer MLP with hidden dimension 512, and evaluate on three datasets of increasing difficulty. Baselines are four hand-designed predefined workflows (Appendix \ref{appendix:llm-workflows}, Figure \ref{fig:workflows}): zero-shot Chain-of-Thought (CoT), reflection prompting, and two combinations.

\begin{table}[t]
	\caption{Accuracy of predefined workflows and the workflow learned by IC-$Q$. On every dataset, IC-$Q$ converges to the highest-performing workflow without observing joint trajectories.}
	\label{tab:lm_acc}
	\vskip 0.05in
	\centering
	\begin{small}
		\begin{tabular}{c|c|cccc|l}
			\hline
			\multirow{2}{*}{Dataset} & \multirow{2}{*}{\begin{tabular}[c]{@{}c@{}}Direct\\ Ans\end{tabular}} & \multicolumn{4}{c|}{Workflow} & \multirow{2}{*}{IC-$Q$ Learned} \\
			&  & CoT & Reflect. & Comb-E & Comb-T &  \\ \hline
			MathQA & 0.299 & 0.833 & 0.825 & \textbf{0.836} & 0.823 & \textbf{0.836 (CoT+Re.)} \\
			MMLU (HS) & 0.415 & 0.863 & 0.859 & \textbf{0.881} & 0.852 & \textbf{0.881 (CoT+Re.)}\\
			MMLU (Elem.) & 0.648 & \textbf{0.960} & 0.944 & 0.958 & 0.950 & \textbf{0.960 (CoT)}\\ \hline
		\end{tabular}
	\end{small}
	\vskip -0.1in
\end{table}

Table \ref{tab:lm_acc} reports accuracy. IC-$Q$ recovers the optimal predefined workflow on every dataset, despite no agent observing joint trajectories -- the policy-correspondence claim of \S\ref{sec:IC-SMDP} cashed out empirically. The optimal workflow is dataset-dependent and IC-$Q$ discovers this without manual tuning: on harder datasets (MathQA, MMLU high-school) it engages the checker, matching Combine-E; on easier data (MMLU elementary) the checker introduces errors and IC-$Q$ correctly converges to simpler CoT. Combine-T (editor receives feedback only via the thinker) is consistently dominated and never selected, showing IC-$Q$ identifies that information loss in the editor's pipeline degrades performance. The four-bit AIS observation is a sixteen-state projection of an unbounded conversation, so $(\varepsilon_\phi, \delta_\phi)$ are nontrivial; their empirical smallness, evidenced by IC-$Q$ matching the centralized oracle, is what enables the algorithm's success.

\subsection{Routing and adaptable agents}

Two further experiments instantiate regimes complementary to the multi-LLM headline. {\em Multi-agent routing} treats $N=100$ agents as routers in randomly generated graphs (Erd\H{o}s-R\'{e}nyi, Barab\'{a}si-Albert, Watts-Strogatz, Chain), each observing only its local neighborhood. The AIS observation is exact ($\varepsilon_\phi = \delta_\phi = 0$, since destination plus detain-flag suffices), so this experiment validates IC-$Q$ in the regime where Theorem \ref{thm:finite-sample-preconfigured}'s AIS term collapses. IC-$Q$ achieves 100\% routing accuracy across all graph distributions (Appendix \ref{appendix:routing}). 
{\em Multi-agent CPU programming} uses six agents (starter, two loaders, ALU, selector, writer) representing CPU components that collaboratively transform an initial memory state into a target state. Both local actions and successor selection are learned -- the adaptable regime of \S\ref{sec:algorithm} that Theorem \ref{thm:finite-sample-preconfigured} does not formally cover. IC-$Q$ achieves $\sim$80\% accuracy on previously unseen target memory states even when trained on as little as 20\% of the integer range (Appendix \ref{appendix:cpu}), evidence that the learned workflows generalize compositionally and that the framework empirically extends beyond the pre-configured regime our theory currently treats.

\section{Conclusion}
We introduced the IC-SMDP as a formal model of decentralized agentic workflow control, designed an asynchronous $Q$-learning algorithm with single-scalar value passing at handoffs (Algorithm \ref{alg:ic-q}), and proved a finite-sample convergence bound decomposing into neural approximation, AIS gap, and mixing-time residual (Theorem \ref{thm:finite-sample-preconfigured}). The framework covers a regime prior decentralized RL theory does not, namely sequential handoffs with interface-limited observation, and yields the first joint asymptotic-and-finite-sample bound for neural $Q$-learning under decentralized partial observability. Empirically, four experiments spanning a synthetic IC-SMDP, multi-LLM mathematical reasoning, multi-agent routing, and adaptable CPU programming cash out the bound's three error terms as distinguishable phenomena and show that IC-$Q$ matches a centralized oracle without any agent observing joint trajectories.
\newpage
\bibliographystyle{plain}
\bibliography{reference,intro_references}

\begin{thebibliography}{10}

\bibitem{anthropic2024mcp}
{Anthropic}.
\newblock Introducing the {M}odel {C}ontext {P}rotocol.
\newblock \url{https://www.anthropic.com/news/model-context-protocol}, 2024.

\bibitem{bacon2017optioncritic}
Pierre-Luc Bacon, Jean Harb, and Doina Precup.
\newblock The option-critic architecture.
\newblock In {\em Proceedings of the {AAAI} Conference on Artificial
  Intelligence}, 2017.

\bibitem{bernstein2002decpomdp}
Daniel~S. Bernstein, Robert Givan, Neil Immerman, and Shlomo Zilberstein.
\newblock The complexity of decentralized control of {M}arkov decision
  processes.
\newblock {\em Mathematics of Operations Research}, 27(4):819--840, 2002.

\bibitem{bhandari2018finite}
Jalaj Bhandari, Daniel Russo, and Raghav Singal.
\newblock A finite time analysis of temporal difference learning with linear
  function approximation.
\newblock In {\em Proceedings of the 31st Conference on Learning Theory},
  volume~75 of {\em Proceedings of Machine Learning Research}, pages
  1691--1692. PMLR, 2018.

\bibitem{bradtke1994smdp}
Steven~J. Bradtke and Michael~O. Duff.
\newblock Reinforcement learning methods for continuous-time {M}arkov decision
  problems.
\newblock In {\em Advances in Neural Information Processing Systems (NeurIPS)},
  1994.

\bibitem{cai2023neural}
Qi~Cai, Zhuoran Yang, Jason~D. Lee, and Zhaoran Wang.
\newblock Neural temporal difference and {Q} learning provably converge to
  global optima.
\newblock {\em Mathematics of Operations Research}, 49(1):619--651, 2023.

\bibitem{chai2025deep}
Jinhang Chai, Elynn Chen, and Jianqing Fan.
\newblock Deep transfer $q$-learning for offline non-stationary reinforcement
  learning.
\newblock {\em arXiv preprint arXiv:2501.04870}, 2025.

\bibitem{chai2025transfer}
Jinhang Chai, Elynn Chen, and Lin Yang.
\newblock Transfer {Q}-learning with composite {MDP} structures.
\newblock In {\em Proceedings of the 42nd International Conference on Machine
  Learning (ICML)}, volume 267 of {\em Proceedings of Machine Learning
  Research}, pages 7089--7106. PMLR, 2025.

\bibitem{chai2026optimistic}
Jinhang Chai, Enpei Zhang, Elynn Chen, and Yujun Yan.
\newblock Optimistic transfer under task shift via {Bellman} alignment.
\newblock {\em arXiv preprint arXiv:2601.21924}, 2026.

\bibitem{chen2026data}
Elynn Chen, Xi~Chen, and Wenbo Jing.
\newblock Data-driven knowledge transfer in batch $q^*$ learning.
\newblock {\em Journal of the American Statistical Association}, 2026.
\newblock Accepted; published online 05 Jan 2026.

\bibitem{chen2025highdim}
Elynn Chen, Xi~Chen, Wenbo Jing, and Xiao Liu.
\newblock High-dimensional linear bandits under stochastic latent
  heterogeneity.
\newblock {\em arXiv preprint arXiv:2502.00423}, 2025.

\bibitem{chen2026transfer}
Elynn Chen, Xi~Chen, and Yi~Zhang.
\newblock Transfer learning for contextual joint assortment-pricing under
  cross-market heterogeneity.
\newblock {\em arXiv preprint arXiv:2603.18114}, 2026.

\bibitem{chen2025transfer}
Elynn Chen, Sai Li, and Michael~I. Jordan.
\newblock Transfer {Q}-learning for finite-horizon {Markov} decision processes.
\newblock {\em Electronic Journal of Statistics}, 19(2):5289--5312, 2025.

\bibitem{chen2024reinforcement}
Elynn Chen, Rui Song, and Michael~I. Jordan.
\newblock Reinforcement learning in latent heterogeneous environments.
\newblock {\em Journal of the American Statistical Association},
  119(548):3113--3126, 2024.

\bibitem{fan2024workflowllm}
Shengda Fan, Xin Cong, Yuepeng Fu, Zhong Zhang, Shuyan Zhang, Yuanwei Liu,
  Yesai Wu, Yankai Lin, Zhiyuan Liu, and Maosong Sun.
\newblock {WorkflowLLM}: Enhancing workflow orchestration capability of large
  language models.
\newblock {\em arXiv preprint arXiv:2411.05451}, 2024.

\bibitem{foerster2018coma}
Jakob Foerster, Gregory Farquhar, Triantafyllos Afouras, Nantas Nardelli, and
  Shimon Whiteson.
\newblock Counterfactual multi-agent policy gradients.
\newblock In {\em Proceedings of the {AAAI} Conference on Artificial
  Intelligence}, 2018.

\bibitem{google2025a2a}
{Google}.
\newblock {A}gent2{A}gent ({A2A}) protocol specification, 2025.

\bibitem{gottweis2025aiscientist}
Juraj Gottweis, Wei-Hung Weng, Alexander Daryin, Tao Tu, Anil Palepu, Petar
  Sirkovic, Artiom Myaskovsky, Felix Weissenberger, Keran Rong, Ryutaro Tanno,
  et~al.
\newblock Towards an {AI} co-scientist.
\newblock {\em arXiv preprint arXiv:2502.18864}, 2025.

\bibitem{hong2023metagpt}
Sirui Hong, Xiawu Zheng, Jonathan Chen, Yuheng Cheng, Jinlin Wang, Ceyao Zhang,
  Zili Wang, Steven Ka~Shing Yau, Zijuan Lin, Liyang Zhou, et~al.
\newblock {MetaGPT}: Meta programming for multi-agent collaborative framework.
\newblock {\em arXiv preprint arXiv:2308.00352}, 2023.
\newblock Cited as Hong et al., 2024 per ICLR publication.

\bibitem{hu2024adas}
Shengran Hu, Cong Lu, and Jeff Clune.
\newblock Automated design of agentic systems.
\newblock {\em arXiv preprint arXiv:2408.08435}, 2024.

\bibitem{huang2023agentcoder}
Dong Huang, Qingwen Bu, Jie~M. Zhang, Michael Luck, and Heming Cui.
\newblock {AgentCoder}: Multi-agent-based code generation with iterative
  testing and optimisation.
\newblock {\em arXiv preprint arXiv:2312.13010}, 2023.

\bibitem{iqbal2019actorattentioncritic}
Shariq Iqbal and Fei Sha.
\newblock Actor-attention-critic for multi-agent reinforcement learning.
\newblock In {\em Proceedings of the 36th International Conference on Machine
  Learning ({ICML})}, volume~97 of {\em Proceedings of Machine Learning
  Research}, pages 2961--2970. PMLR, 2019.

\bibitem{kao2022common}
Hsu Kao and Vijay Subramanian.
\newblock Common information based approximate state representations in
  multi-agent reinforcement learning.
\newblock In {\em Artificial Intelligence and Statistics (AISTATS)}, 2022.

\bibitem{kara2022finite}
Ali~Devran Kara and Serdar Y{\"u}ksel.
\newblock Convergence of finite memory {Q}-learning for {POMDP}s and near
  optimality of learned policies under filter stability.
\newblock {\em Mathematics of Operations Research}, 48(4):2066--2093, 2022.

\bibitem{lowe2017maddpg}
Ryan Lowe, Yi~I. Wu, Aviv Tamar, Jean Harb, Pieter Abbeel, and Igor Mordatch.
\newblock Multi-agent actor-critic for mixed cooperative-competitive
  environments.
\newblock In {\em Advances in Neural Information Processing Systems (NeurIPS)},
  2017.

\bibitem{nair2005networked}
Ranjit Nair, Pradeep Varakantham, Milind Tambe, and Makoto Yokoo.
\newblock Networked distributed {POMDPs}: A synthesis of distributed constraint
  optimization and {POMDPs}.
\newblock In {\em Proceedings of the 20th National Conference on Artificial
  Intelligence ({AAAI})}, pages 133--139, 2005.

\bibitem{oliehoek2016decpomdp}
Frans~A. Oliehoek and Christopher Amato.
\newblock {\em A Concise Introduction to Decentralized {POMDP}s}.
\newblock Springer, 2016.

\bibitem{oliehoek2008exploiting}
Frans~A. Oliehoek, Matthijs T.~J. Spaan, Shimon Whiteson, and Nikos Vlassis.
\newblock Exploiting locality of interaction in factored {Dec-POMDPs}.
\newblock In {\em Proceedings of the 7th International Joint Conference on
  Autonomous Agents and Multiagent Systems ({AAMAS})}, pages 517--524, 2008.

\bibitem{precup2000options}
Doina Precup.
\newblock {\em Temporal Abstraction in Reinforcement Learning}.
\newblock PhD thesis, University of Massachusetts Amherst, 2000.

\bibitem{rashid2020qmix}
Tabish Rashid, Mikayel Samvelyan, Christian~Schroeder De~Witt, Gregory
  Farquhar, Jakob Foerster, and Shimon Whiteson.
\newblock Monotonic value function factorisation for deep multi-agent
  reinforcement learning.
\newblock {\em Journal of Machine Learning Research}, 21(178):1--51, 2020.

\bibitem{sinha2024periodic}
Amit Sinha, Matthieu Geist, and Aditya Mahajan.
\newblock Periodic agent-state based {Q}-learning for {POMDP}s.
\newblock In {\em Advances in Neural Information Processing Systems (NeurIPS)},
  2024.

\bibitem{sinha2024agentstate}
Amit Sinha and Aditya Mahajan.
\newblock Agent-state based policies in {POMDP}s: Beyond belief-state {MDP}s.
\newblock In {\em IEEE Conference on Decision and Control (CDC)}, 2024.

\bibitem{yue2025daao}
Jinwei Su, Qizhen Lan, Yinghui Xia, Lifan Sun, Weiyou Tian, Tianyu Shi, and
  Lewei He.
\newblock Difficulty-aware agentic orchestration for query-specific multi-agent
  workflows.
\newblock pages 2060--2070, 2026.

\bibitem{subramanian2022ais}
Jayakumar Subramanian, Amit Sinha, Raihan Seraj, and Aditya Mahajan.
\newblock Approximate information state for approximate planning and
  reinforcement learning in partially observed systems.
\newblock {\em Journal of Machine Learning Research}, 23(12):1--83, 2022.
\newblock arXiv:2010.08843.

\bibitem{sunehag2018vdn}
Peter Sunehag, Guy Lever, Audrunas Gruslys, Wojciech~Marian Czarnecki, Vinicius
  Zambaldi, Max Jaderberg, Marc Lanctot, Nicolas Sonnerat, Joel~Z. Leibo, Karl
  Tuyls, et~al.
\newblock Value-decomposition networks for cooperative multi-agent learning
  based on team reward.
\newblock In {\em Autonomous Agents and Multi-Agent Systems (AAMAS)}, 2018.

\bibitem{sutton1999options}
Richard~S. Sutton, Doina Precup, and Satinder Singh.
\newblock Between {MDP}s and semi-{MDP}s: A framework for temporal abstraction
  in reinforcement learning.
\newblock {\em Artificial Intelligence}, 112(1--2):181--211, 1999.

\bibitem{wu2023autogen}
Qingyun Wu, Gagan Bansal, Jieyu Zhang, Yiran Wu, Beibin Li, Erkang Zhu,
  Li~Jiang, Xiaoyun Zhang, Shaokun Zhang, Jiale Liu, Ahmed~Hassan Awadallah,
  Ryen~W. White, Doug Burger, and Chi Wang.
\newblock {AutoGen}: Enabling next-gen {LLM} applications via multi-agent
  conversation.
\newblock {\em arXiv preprint arXiv:2308.08155}, 2023.

\bibitem{xu2020finite}
Pan Xu and Quanquan Gu.
\newblock A finite-time analysis of q-learning with neural network function
  approximation.
\newblock In {\em Proceedings of the 37th International Conference on Machine
  Learning}, volume 119 of {\em Proceedings of Machine Learning Research},
  pages 10555--10565. PMLR, 2020.

\bibitem{yang2025agentnet}
Yingxuan Yang, Huacan Chai, Shuai Shao, Yuanyi Song, Siyuan Qi, Renting Rui,
  and Weinan Zhang.
\newblock {AgentNet}: Decentralized evolutionary coordination for {LLM}-based
  multi-agent systems.
\newblock In {\em Advances in Neural Information Processing Systems (NeurIPS)},
  2025.
\newblock arXiv:2504.00587.

\bibitem{yao2023react}
Shunyu Yao, Jeffrey Zhao, Dian Yu, Nan Du, Izhak Shafran, Karthik Narasimhan,
  and Yuan Cao.
\newblock {ReAct}: Synergizing reasoning and acting in language models.
\newblock {\em International Conference on Learning Representations (ICLR)},
  2023.

\bibitem{zhang2024aflow}
Jiayi Zhang, Jinyu Xiang, Zhaoyang Yu, Fengwei Teng, Xiongwei Chen, Jiaqi Chen,
  Mingchen Zhuge, Xin Cheng, Sirui Hong, Jinlin Wang, et~al.
\newblock {AFlow}: Automating agentic workflow generation.
\newblock In {\em International Conference on Learning Representations (ICLR)},
  2025.
\newblock arXiv:2410.10762.

\bibitem{zhang2025transfer}
Yi~Zhang, Elynn Chen, and Yujun Yan.
\newblock Transfer faster, price smarter: Minimax dynamic pricing under
  cross-market preference shift.
\newblock In {\em Advances in Neural Information Processing Systems (NeurIPS)},
  2025.
\newblock Spotlight; arXiv:2505.17203.

\bibitem{zhou2025prior}
Runlin Zhou, Chixiang Chen, and Elynn Chen.
\newblock Prior-aligned meta-{RL}: {Thompson} sampling with learned priors and
  guarantees in finite-horizon {MDP}s.
\newblock {\em arXiv preprint arXiv:2510.05446}, 2025.

\bibitem{zhuge2024gptswarm}
Mingchen Zhuge, Wenyi Wang, Louis Kirsch, Francesco Faccio, Dmitrii Khizbullin,
  and J{\"u}rgen Schmidhuber.
\newblock {GPTSwarm}: Language agents as optimizable graphs.
\newblock In {\em Proceedings of the 41st International Conference on Machine
  Learning (ICML)}, 2024.

\end{thebibliography}
\end{document}